\documentclass[preprint,12pt,authoryear]{elsarticle}

\usepackage{natbib}
\usepackage{graphicx}
\usepackage{amssymb}
\usepackage{amsthm}
\usepackage{amsmath}
\usepackage{amsfonts}
\usepackage{subfigure}
\usepackage[colorlinks]{hyperref}

\usepackage{lineno}
\newtheorem{definition}{Definition}
\newtheorem{thm}{Theorem} 
\newtheorem{lem}{Lemma}

\journal{Expert Systems with Applications}
\begin{document}
\begin{frontmatter}

\title{A divide and conquer method for symbolic regression \tnoteref{mytitlenote}}

\tnotetext[mytitlenote]{This work has been supported by the National Natural Science Foundation of China (Grant No. 11532014).}



\author[lhd]{Changtong Luo\corref{cor1}}
\ead{luo@imech.ac.cn}
\author[lhd,ucas]{Chen Chen}
\ead{chenchen@imech.ac.cn}
\author[lhd,ucas]{Zonglin Jiang}
\ead{zljiang@imech.ac.cn}

\cortext[cor1]{Corresponding author}

\address[lhd]{State Key Laboratory of High Temperature Gas Dynamics, Institute of Mechanics,\\ Chinese Academy of Sciences, Beijing 100190, China}

\address[ucas]{School of Engineering Sciences, University of Chinese Academy of Sciences,\\ Beijing, 100049, China}
\begin{abstract}
Symbolic regression aims to find a function that best explains the relationship between independent variables and the objective value based on a given set of sample data. Genetic programming (GP) is usually considered as an appropriate method for the problem since it can optimize functional structure and coefficients simultaneously. However, the convergence speed of GP might be too slow for large scale problems that involve a large number of variables. Fortunately, in many applications, the target function is separable or partially separable. This feature motivated us to develop a new method, divide and conquer (D\&C), for symbolic regression, in which the target function is divided into a number of sub-functions and the sub-functions are then determined by any of a GP algorithm. The separability is probed by a new proposed technique, Bi-Correlation test (BiCT). D\&C powered GP has been tested on some real-world applications, and the study shows that D\&C can help GP to get the target function much more rapidly.


\end{abstract}

\end{frontmatter}


\section{Introduction}
\label{secIntro}

Symbolic regression (SR) is a data-driven modeling method which aims to find a function that best explains the relationship between independent variables and the objective value based on a given set of sample data \citep{Schmidt2009}. Genetic programming (GP) is usually considered as a good candidate for SR since it does not impose a priori assumptions and can optimize function structure and coefficients simultaneously. However, the convergence speed of GP might be too slow for large scale problems that involve a large number of variables.

Many efforts have been devoted trying to improve the original GP \citep{Koza92} in several ways. Some works suggest replacing its tree-based  method, with an integer string (Grammar Evolution) \citep{ONeill01}, or a parse matrix (Parse-Matrix Evolution) \citep{Luo2012}. These techniques can simplify the coding and decoding process but help little on improving the convergence speed. Some other works suggest confining its search space to generalized linear space, for example, Fast Function eXtraction \citep{McConaghy2011}, and Elite Bases Regression \citep{Chen2017}. These techniques can accelerate the convergence speed of GP, even by orders of magnitude. However, the speed is gained at the sacrifice of losing the generality, that is, the result might be only a linear approximation of the target function. 

Fortunately, in many applications, the target function is separable or partially separable (see section 2 for definitions). For example, in gas dynamics \citep{Anderson2006}, the heat flux coefficient $S_t$ of a flat plate could be formulated as 
\begin{equation}
\label{equ1}
S_t=2.274 \sin(\theta) \sqrt{\cos(\theta)}/\sqrt{Re_x},
\end{equation}
and the heat flux $q_s$ at the stagnation point of a sphere as
\begin{equation}
\label{equ2}
q_s=1.83\times10^{-4}v^3\sqrt{\rho/R}(1-h_w/h_s).
\end{equation}
In equation (\ref{equ1}), the two independent variables, $\theta$ and $Re_x$, are both separable. In equation (\ref{equ2}), the first three variables $v$, $\rho$, and $R$ are all separable, and the last two variables, $h_w$ and $h_s$, are not separable, but their combination ($h_w$, $h_s$) is separable. The function in equation (\ref{equ2}) is considered partially separable in this paper.

The feature of separability will be used in this paper to accelerate the optimization process of symbolic regression. Some basic concepts on function separability are defined in Section \ref{secDef}. Section \ref{secDAC} describes the overall work flow of the proposed method, divide and conquer. Section \ref{secBiCT} presents a special technique, bi-correlation test (BiCT), to determine the separability of a function. Numerical results are given in Section \ref{secNResults}, and the concluding remarks are drawn in Section \ref{SecConclusion}. 

\section{Basic concepts}
\label{secDef}
The proposed method in this paper is based on a new concept referred to as partial separability. 
It has something in common with existing separability definitions such as reference \citep{Berenguel2013} and \citep{dAvezac2011}, but is not exactly the same. To make it clear and easy to understand, we begin with some illustrative examples. The functions as follows could all be regarded as partially separable:
\begin{equation}
\label{eqToy1}
z=0.8+0.6*(\boxed{u^2+\cos(u)})+\boxed{\sin(v+w)*(v-w)};
\end{equation}
\begin{equation}
\label{eqToy2}
z=0.8+0.6*(\boxed{u^2+\cos(u)})-\boxed{\sin(v+w)*(v-w)};
\end{equation}
\begin{equation}
\label{eqToy3}
z=0.8+0.6*(\boxed{u^2+\cos(u)})*\boxed{\sin(v+w)*(v-w)}; 
\end{equation}
where the boxed frames are used to indicate sub-functions, $u$ is separable with respect to $z$, while $v$ and $w$ themselves are not separable, but their combination $(v, w)$ is considered separable. A simple example of non-separable function is $f(\boldsymbol x)=\sin(x_1+x_2+x_3)$.

\bigskip
More precisely, the separability could be defined as follows.

\begin{definition}
\label{def1}
A scalar function with $n$ continuous variables $f(\boldsymbol x)$ ($f:{\mathbb{R}^n} \mapsto \mathbb{R}$, $\boldsymbol x \in {\mathbb{R}^n}$) is said to be partially separable if and only if it can be rewritten as
\begin{equation}
\label{eq3}
f(\boldsymbol x) = {c_0} \otimes_1 {\varphi _1}({I_1}x) \otimes_2 {\varphi _2}({I_2}x) \otimes_3  \cdots  \otimes_m {\varphi _m}({I_m}x)
\end{equation}
where the binary operator $\otimes_i$ could be plus ($+$), minus ($-$), times($\times$).  ${I_i} $ is a sub-matrix of the identity matrix, and ${I_i} \in {\mathbb{R}^{{n_i} \times n}}$. The set ${\text{\{ }}{I_1},{I_2}, \cdots ,{I_m}{\text{\} }}$ is a partition of the identity matrix $I \in {\mathbb{R}^{n \times n}}$, $\sum\limits_{i = 1}^m {{n_i}}  = n$. The sub-function ${\varphi _i}$ is a scalar function such that
${\varphi _i}:{\mathbb{R}^{{n_i}}} \mapsto \mathbb{R}$. Otherwise the function is said to be non-separable.
\end{definition}

In this definition, the binary operator, division ($/$), is not included in $\otimes$ for simplicity. However, this does not affect much of its generality, since the sub-functions are not preset, and can be transformed as $\tilde{\varphi}_i(\cdot)$=  $1/\varphi_i(\cdot)$ if only $\varphi_i(\cdot)\neq 0$. 

A special case is that all variables are separable, which could be defined as follows. 

\begin{definition}
A scalar function with $n$ continuous variables $f(\boldsymbol x)$ ($f:{\mathbb{R}^n} \mapsto \mathbb{R}$, $\boldsymbol x \in {\mathbb{R}^n}$) is said to be completely separable if and only if it can be rewritten as equation (\ref{eq3}) and $n_i=1$ for all $i=1, 2, \cdots, m$.
\end{definition}

\section{Divide and conquer method}
\label{secDAC}
As above mentioned, many practical problems have the feature of separability. To make use of this feature to accelerate the optimization process of genetic programming, a new method, divide and conquer (D\&C), is proposed. It works as follows. 

First, a separability detection process is carried out to find out
whether the concerned problem is separable (at least partial separable) or not. The variables are identified one by one, and then their combinations. Once it (or the variable combination) is identified as separable, a sub-function $\varphi_i(x_i)$ (or $\varphi_i(I_i {\boldsymbol x})$ for variable combinations) will be assigned. In this way, the structure of target function $f(\boldsymbol x)$ could be divided into a set of sub-functions based on the separability information: $\varphi_i(I_i \boldsymbol x)$, $i= 1, 2, \cdots, m$.

Then, the sub-functions $\varphi_i(I_i \boldsymbol x)$ ($i= 1, 2, \cdots, m$) are optimized and determined one by one, using any of genetic programming algorithms, including classical GP, Grammatical Evolution \citep{ONeill01} and Parse-Matrix Evolution \citep{Luo2012}. When optimizing one sub-function, variables not involved ($(I-I_i) \boldsymbol x$) are fixed as constants. That is, only a small number of variables ($I_i \boldsymbol x$, which is only a subset of $\{ x_1, x_2, \cdots, x_n \}$) need be considered. 
This means the sub-function determination should be much easier than evolving the target function $f(\boldsymbol x)$ directly. 

\textbf{For example, in Equation (\ref{eqToy1}), Equation (\ref{eqToy2}), or Equation (\ref{eqToy3}), the sub-function $\varphi_1(u)=u^2+\cos(u)$ (or $\varphi_2(v, w)=\sin(v+w)*(v-w)$) has less number of variables and complexity than the original function. Thus, optimizing them one by one is much easier for GP. }

Finally, these sub-functions are properly combined to form the target function, which is referred to as a function recover process. This process is rather simple, and all traditional regression algorithms are qualified to accomplish this mission.

The work flow of D\&C could be described in Figure \ref{fig:workflow}.
\begin{figure}
\centering
\includegraphics[width=0.6\linewidth]{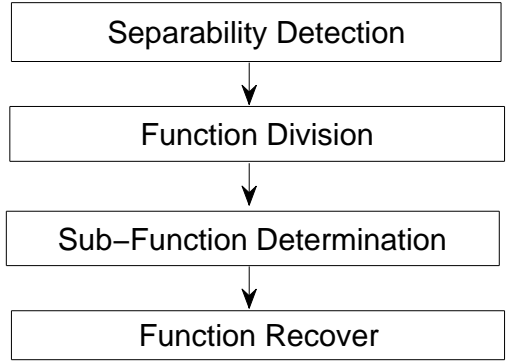}
\caption{Work flow of divide and conquer for symbolic regression}
\label{fig:workflow}
\end{figure}

\section{Bi-correlation test}
\label{secBiCT}
\subsection{Description}
The main idea of the proposed divide and conquer (D\&C) method is to make use of the separability feature to simplify the search process. Therefore, the most important and fundamental step (the separability detection process in Fig. 1) is to determine whether the concerned problem is separable (at least partial separable) or not. To fulfill this task, a special technique, Bi-correlation test (BiCT), is provided in this section. 

Consider independent variables $x_1, x_2, \cdots, x_n$ and the dependent $f$ as n+1 random variables, and the known data as sample points. Recall that the linear relation and  correlation coefficient of two random variables has the following relations.

\begin{lem}
\label{lemma0}
The two random variables $\xi$ and $\eta$ are linearly related with correlation coefficient 1 (i.e., $\rho_{\xi\eta}=1$) if and only if there exists two constants $a$, $b$ ($b\neq 0$), such that $P\{\eta=a+b\xi\}=1$.
\end{lem}

The correlation coefficient $ \rho$ could be estimated by the sample correlation coefficient $r$, which is defined as follows.
$$r=\frac{1}{n-1}  \sum\limits_{i = 1}^{N} \frac{(\xi _i - \bar \xi )} {\sigma _\xi } \cdot \frac{(\eta _i - \bar \eta )} {\sigma _\eta }$$
\noindent where $N$ is the number of observations in the sample set, $\sum$ is the summation symbol, $\xi_i$ is the $\xi$ value for observation i, $\bar \xi$ is the sample mean of $\xi$, $\eta_i$ is the $\eta$ value for observation i, $\bar \eta$ is the sample mean of $\eta$, $\sigma _\xi$ is the sample standard deviation of $\xi$, and $\sigma _\eta$ is the sample standard deviation of $\eta$.

Only continuous model functions are considered in this paper. As a result, the conclusion of Lemma \ref{lemma0} could be simplified as follows. 

\begin{itemize}
\item The two random variables ${\boldsymbol f}_A$ and ${\boldsymbol f}_B$ are linearly related (${\boldsymbol f}_B=a+b{\boldsymbol f}_A$) 
if and only if the sample correlation coefficient $r=1$ for any given sample set. 
\end{itemize}

Studies shows that the functional separability defined in the above section (See equation \ref{eq3}) could be observed with random sampling and linear correlation techniques. 

Without the loss of generality, a simple function with three variables ($f(\boldsymbol x)=f(x_1, x_2, x_3), x_i \in [a_i, b_i], i=1,2,3$) is considered to illustrate the implementation of the bi-correlation test. To find out whether the first variable $x_1$ is separable, two correlation tests are needed.

The first correlation test is carried out as follows. A set of random sample points in $[a_1, b_1]$ are generated, then these points are extended to three dimensional space with the rest variables ($x_2$ and $x_3$) fixed to a point $A$. We get a vector of function values ${\boldsymbol f}^{(A)}={\boldsymbol f}{(x_1,A)}=(f_{1}^A, f_{2}^A, \cdots, f_{N}^A)$, where $N$ is the number of sample points. Then these points are extended to three dimensional space with fixed $x_2$ and $x_3$ to another point $B$, We get another vector ${\boldsymbol f}^{(B)}={\boldsymbol f}{(x_1,B)}=(f_{1}^B, f_{2}^B, \cdots, f_{N}^B)$. It is obviously that the two vectors $f^{(A)}$ and $f^{(B)}$ will be linearly related if $x_1$ is separable. However, it is easy to show that this linear relation could NOT ensure its separability.

Then it comes to the second correlation test. Another set of random sample points in $[a_2, b_2]\times[a_3, b_3]$ are generated, then these points are extended to three dimensional space with the rest variable(s) ($x_1$ in this case) fixed to a point $C$, and get a vector ${\boldsymbol f}^{(C)}={\boldsymbol f}{(C, x_2, x_3)}$. Similarly, another vector ${\boldsymbol f}^{(D)}={\boldsymbol f}{(D, x_2, x_3)}$ is obtained. Again, the two vectors ${\boldsymbol f}^{(C)}$ and ${\boldsymbol f}^{(D)}$ needs to be linearly related to ensure the separability of  $x_1$.

\subsection{Proposition} 

Without the loss of generality, suppose we have a scalar function ${f}(\boldsymbol x)$ with $n$ continuous variables  (${f}:{\mathbb{R}^n} \mapsto \mathbb{R}$, $\boldsymbol x \in \Omega\subset {\mathbb{R}^n}$, and $\Omega =[a_1, b_1]\times[a_2, b_2]\times\cdots\times[a_n, b_n]$), and need to find out whether the first $m$ variable combination $(x_1, x_2,\cdots, x_m)$ are separable. Let the matrix $X_1$ be a set of $N$ random sample points from the subset $[a_1, b_1]\times[a_2, b_2]\times\cdots\times[a_m, b_m] \subset\mathbb{R}^{m}$, and
$${X_1} = \left[ {\begin{array}{*{20}{c}}
  {{x_{1}^{(1)}}}&{{x_{2}^{(1)}}}& \cdots &{{x_{m}^{(1)}}} \\ 
 {{x_{1}^{(2)}}}&{{x_{2}^{(2)}}}& \cdots &{{x_{m}^{(2)}}} \\ 
   \cdots & \cdots & \cdots & \cdots  \\ 
  {{x_{1}^{(N)}}}&{{x_{2}^{(N)}}}& \cdots &{{x_{m}^{(N)}}} 
\end{array}} \right].$$  
The rest variables $x_{m+1}, x_{m+2}, \cdots, x_n$ are fixed to two given points A and B in the subset $[a_{m+1}, b_{m+1}]\times[a_{m+2}, b_{m+2}]\times\cdots\times[a_n, b_n]\subset \mathbb{R}^{n-m}$, i.e., $\boldsymbol x_A=({{x_{m + 1}}^{(A)}},{{x_{m + 2}}^{(A)}}, \cdots,{{x_n}^{(A)}})$, $\boldsymbol x_B=({{x_{m + 1}}^{(B)}},{{x_{m + 2}}^{(B)}}, \cdots,{{x_n}^{(B)}})$.  \\
Let the matrix
${X_2}^{(A)} = \left[ \begin{gathered}
  1 \hfill \\
  1 \hfill \\
   \cdots  \hfill \\
  1 \hfill \\ 
\end{gathered}  \right] \boldsymbol x_A = \left[ {\begin{array}{*{20}{c}}
  {{x_{m + 1}}^{(A)}}&{{x_{m + 2}}^{(A)}}& \cdots &{{x_n}^{(A)}} \\ 
  {{x_{m + 1}}^{(A)}}&{{x_{m + 2}}^{(A)}}& \cdots &{{x_n}^{(A)}} \\ 
   \cdots & \cdots & \cdots & \cdots  \\ 
  {{x_{m + 1}}^{(A)}}&{{x_{m + 2}}^{(A)}}& \cdots &{{x_n}^{(A)}} 
\end{array}} \right]
$, and ${X_2}^{(B)} = \left[ \begin{gathered}
  1 \hfill \\
  1 \hfill \\
   \cdots  \hfill \\
  1 \hfill \\ 
\end{gathered}  \right] \boldsymbol x_B$. \\
Let the extended matrix ${X_A} = \left[ {\begin{array}{*{20}{c}}
  {{X_1}}&{{X_2}^{(A)}} 
\end{array}} \right]$, and ${X_B} = \left[ {\begin{array}{*{20}{c}}
  {{X_1}}&{{X_2}^{(B)}} 
\end{array}} \right]$. \\
Let ${\boldsymbol f_A}$ be the vector of which the $i$-th element is the function value of the $i$-th row of matrix ${X_A}$, i.e., ${\boldsymbol f_A} = f({X_A})$, and ${\boldsymbol f_B}$ is similarly defined, ${\boldsymbol f_B} = f({X_B})$.

\begin{lem}
\label{lemma1}
The two vectors ${\boldsymbol f_A} $ and ${\boldsymbol f_B} $ are linearly related if the function $f(\boldsymbol x)$ is separable with respect to the first $m$ variable combination $(x_1, x_2,\cdots, x_m)$.
\end{lem}
\begin{proof}
Since the first $m$ variable combination $(x_1, x_2,\cdots, x_m)$ are separable, from definition \ref{def1}, we have $f(\boldsymbol x)=\varphi_1(x_1, x_2,\cdots, x_m) \otimes \varphi_2(x_{m+1}, x_{m+2},\cdots, x_n)$. Accordingly, the vector ${\boldsymbol f_A} = f({X_A})=\varphi_1(X_1) \otimes \varphi_2({X_2^{A}})= \varphi_1(X_1)\otimes k_A$, where $ \otimes$ is a component-wise binary operation, and $k_A=\varphi_2(\boldsymbol x_A)$ is a scalar. Similarly, the vector ${\boldsymbol f_B}= \varphi_1(X_1)\otimes k_B $. As a result, 
\[{\boldsymbol{f}}_A=\left\{ {\begin{array}{*{20}{l}}
  {{k_A}/{k_B} \cdot {{\boldsymbol{f}}_B}}&{{\text{if }} \otimes {\text{ is times}}} \\ 
  {{k_A} - {k_B} + {{\boldsymbol{f}}_B}}&{{\text{if }} \otimes {\text{ is plus}}} \\ 
  {{k_B} - {k_A} + {{\boldsymbol{f}}_B}}&{{\text{if }} \otimes {\text{ is minus}}} 
\end{array}} \right.\] 
which means the two vectors ${\boldsymbol f_A} $ and ${\boldsymbol f_B} $ are linearly related. 
\end{proof}

On the other hand, if the first $m$ variables are fixed to two given points C and D, and the rest of $n-m$ variables are randomly sampled. A similar proposition could be concluded as follows. Let\\ 
${X_1}^{(C)}  = \left[ \begin{gathered}
  1 \hfill \\
  1 \hfill \\
   \cdots  \hfill \\
  1 \hfill \\ 
\end{gathered}  \right]\left[ {\begin{array}{*{20}{c}}
  {{x_1}^{(C)}}&{{x_2}^{(C)}}& \cdots &{{x_m}^{(C)}} 
\end{array}} \right] =\left[ {\begin{array}{*{20}{c}}
  {{x_1}^{(C)}}&{{x_2}^{(C)}}& \cdots &{{x_m}^{(C)}} \\ 
  {{x_1}^{(C)}}&{{x_2}^{(C)}}& \cdots &{{x_m}^{(C)}} \\ 
   \cdots & \cdots & \cdots & \cdots  \\ 
  {{x_1}^{(C)}}&{{x_2}^{(C)}}& \cdots &{{x_m}^{(C)}} 
\end{array}} \right]$, \\
${X_1}^{(D)}  = \left[ \begin{gathered}
  1 \hfill \\
  1 \hfill \\
   \cdots  \hfill \\
  1 \hfill \\ 
\end{gathered}   \right]\left[ {\begin{array}{*{20}{c}}
  {{x_1}^{(D)}}&{{x_2}^{(D)}}& \cdots &{{x_m}^{(D)}} 
\end{array}} \right]
 $, \\
$
{X_2} = \left[ {\begin{array}{*{20}{c}}
  {{x_{m + 1}^{(1)}}}&{{x_{m + 2}^{(1)}}}& \cdots &{{x_{n}^{(1)}}} \\ 
  {{x_{m + 1}^{(2)}}}&{{x_{m + 2}^{(2)}}}& \cdots &{{x_{n}^{(2)}}} \\ 
   \cdots & \cdots & \cdots & \cdots  \\ 
  {{x_{m + 1}^{(N)}}}&{{x_{m + 2}^{(N)}}}& \cdots &{{x_{n}^{(N)}}} 
\end{array}} \right],
$
the $N\times n$ matrix 
${X_C} = \left[ {\begin{array}{*{20}{c}}
  {{X_1}^{(C)}}&{{X_2}} 
\end{array}} \right]$, and ${X_D} = \left[ {\begin{array}{*{20}{c}}
  {{X_1}^{(D)}}&{{X_2}} 
\end{array}} \right]$. Let ${\boldsymbol f_C}$ be the vector of which the $i$-th element is the function value of the $i$-th row of matrix ${X_C}$, i.e., ${\boldsymbol f_C} = f({X_C})$, and ${\boldsymbol f_D}$ is similarly defined, ${\boldsymbol f_D} = f({X_D})$.

\begin{lem}
\label{lemma2}
The two vectors ${\boldsymbol f_C} $ and ${\boldsymbol f_D} $ are linearly related if the function $f(\boldsymbol x)$ is separable with respect to the first $m$ variable combination $(x_1, x_2,\cdots, x_m)$.
\end{lem}
\begin{proof}
Since the first $m$ variable combination $(x_1, x_2,\cdots, x_m)$ are separable, from definition \ref{def1}, we have $f(\boldsymbol x)=\varphi_1(x_1, x_2,\cdots, x_m) \otimes \varphi_2(x_{m+1}, x_{m+2},\cdots, x_n)$. Accordingly, the vector ${\boldsymbol f_C} = f({X_C})=\varphi_1(X_1^{C}) \otimes \varphi_2({X_2})=k_C \otimes\varphi_2(X_2)$, where $ \otimes$ is a component-wise binary operation, and the scalar $k_C=\varphi_1(\boldsymbol x_C)$. Similarly, the vector ${\boldsymbol f_D}=k_D \otimes\varphi_2(X_2)$. As a result, 
\[{\boldsymbol{f}}_C=\left\{ {\begin{array}{*{20}{l}}
  {{k_C}/{k_D} \cdot {{\boldsymbol{f}}_D}}&{{\text{if }} \otimes {\text{ is times}}} \\ 
  {{k_C} - {k_D} + {{\boldsymbol{f}}_D}}&{{\text{if }} \otimes {\text{ is plus or minus}}} 
\end{array}} \right.\] 
which means the two vectors ${\boldsymbol f_C} $ and ${\boldsymbol f_D} $ are linearly related. 
\end{proof}

The above lemmas show that two function-value vectors must be linearly related if the target function has the separability feature, while the separable variables (or their complement variables) are fixed. These are necessary conditions for the separability identification of target function. The sufficient and necessary conditions are given as follows.

\begin{thm}
\label{thm1}
The function $f(\boldsymbol x)$ is separable with respect to the first $m$ variable combination $(x_1, x_2,\cdots, x_m)$ if and only if both of the flowing statements are true.\\
(1) Any two function-value vectors with fixed $(x_1, x_2,\cdots, x_m)$ are linearly related; \\
(2) Any two function-value vectors with fixed $(x_{m+1}, x_{m+2},\cdots, x_n)$ are linearly related.
\end{thm}



\begin{proof}
		From Lemma \ref{lemma1}, and Lemma \ref{lemma2}, we can conclude that the necessary conditions of the theorem hold. The sufficient conditions can be proved by contradiction. Suppose the separable form $f\left( \boldsymbol{x}  \right) = {\varphi _1}\left( {{x_1},{x_2}, \cdots ,{x_m}} \right) \otimes {\varphi _2}\left( {{x_{m + 1}},{x_{m + 2}} \cdots ,{x_n}} \right)$ can not be derived from the above two conditions. Thus, there is at least one non-separable variable presented in both sub-functions, ${\varphi _1}$ and ${\varphi _2}$. Without loss of generality, we assume ${x_m}$ to be this non-separable variable. That is,
		\begin{equation}
        \label{assume}
		f\left( \boldsymbol{x} \right) = {\varphi _1}\left( {{x_1},{x_2}, \cdots ,{x_m}} \right) \otimes {\varphi _2}\left( {{x_m},{x_{m + 1}},{x_{m + 2}} \cdots ,{x_n}} \right).
		\end{equation}
		Similarly, the process of sampling for the first correlation test can be given as
		\begin{equation}
		{X_1} = \left[ {\begin{array}{*{20}{c}}
			{x_1^{\left( 1 \right)}}&{x_2^{\left( 1 \right)}}& \cdots &{x_m^{\left( 1 \right)}} \\ 
			{x_1^{\left( 2 \right)}}&{x_2^{\left( 2 \right)}}& \cdots &{x_m^{\left( 2 \right)}} \\ 
			\vdots & \vdots &{}& \vdots  \\ 
			{x_1^{\left( N \right)}}&{x_2^{\left( N \right)}}& \cdots &{x_m^{\left( N \right)}} 
			\end{array}} \right],
		\end{equation}
		
		\begin{equation}
		\tilde X_2^{\left( A \right)} = \left[ {\begin{array}{*{20}{c}}
			{x_m^{\left( 1 \right)}}&{x_{m + 1}^{\left( A \right)}}&{x_{m + 2}^{\left( A \right)}}& \cdots &{x_n^{\left( A \right)}} \\ 
			{x_m^{\left( 2 \right)}}&{x_{m + 1}^{\left( A \right)}}&{x_{m + 2}^{\left( A \right)}}& \cdots &{x_m^{\left( A \right)}} \\ 
			\vdots & \vdots & \vdots &{}& \vdots  \\ 
			{x_m^{\left( N \right)}}&{x_{m + 1}^{\left( A \right)}}&{x_{m + 2}^{\left( A \right)}}& \cdots &{x_m^{\left( A \right)}} 
			\end{array}} \right],
		\end{equation}
		and
		\begin{equation}
		\tilde X_2^{\left( B \right)} = \left[ {\begin{array}{*{20}{c}}
			{x_m^{\left( 1 \right)}}&{x_{m + 1}^{\left( B \right)}}&{x_{m + 2}^{\left( B \right)}}& \cdots &{x_n^{\left( B \right)}} \\ 
			{x_m^{\left( 2 \right)}}&{x_{m + 1}^{\left( B \right)}}&{x_{m + 2}^{\left( B \right)}}& \cdots &{x_m^{\left( B \right)}} \\ 
			\vdots & \vdots & \vdots &{}& \vdots  \\ 
			{x_m^{\left( N \right)}}&{x_{m + 1}^{\left( B \right)}}&{x_{m + 2}^{\left( B \right)}}& \cdots &{x_m^{\left( B \right)}} 
			\end{array}} \right].
		\end{equation}
		Let the extended matrix ${X'_A} = \left[ {\begin{array}{*{20}{c}}
			{{X_1}}&{\tilde X_2^{\left( A \right)}} 
			\end{array}} \right]$, and ${X'_B} = \left[ {\begin{array}{*{20}{c}}
			{{X_1}}&{\tilde X_2^{\left( B \right)}} 
			\end{array}} \right]$. Thus, 		
		
\begin{equation}
		{\boldsymbol{f}'_A}{\text{ = }}f \left( {{{X'}_A}} \right){\text{ = }}{\varphi _1}\left( {{X_1}} \right) \otimes {\varphi _2}\left( {\tilde X_2^{\left( A \right)}} \right){\text{ = }}{\varphi _1}\left( {{X_1}} \right) \otimes \boldsymbol \alpha ,
\end{equation}
and
\begin{equation}
		{\boldsymbol{f}'_B}{\text{ = }}f \left( {{{X'}_B}} \right){\text{ = }}{\varphi _1}\left( {{X_1}} \right) \otimes {\varphi _2}\left( {\tilde X_2^{\left( B \right)}} \right){\text{ = }}{\varphi _1}\left( {{X_1}} \right) \otimes \boldsymbol \beta 
\end{equation}
		are defined, where $\boldsymbol \alpha$ and $\boldsymbol \beta$ are function-value vectors of ${\varphi _2}\left( {\tilde X_2^{\left( A \right)}} \right)$ and ${\varphi _2}\left( {\tilde X_2^{\left( B \right)}} \right)$ respectively. As a result,
\begin{equation}
{\boldsymbol{f}'_A}{\text{ = }}\left\{ 
{\begin{array}{*{20}{l}}
  \boldsymbol \gamma  \cdot {{\boldsymbol{f}'}_B} & {\text{  if } \otimes \text{ is times}} \\ 
  \boldsymbol \alpha  - \boldsymbol \beta  + {{\boldsymbol{f}'}_B} & {\text{  if } \otimes \text{ is plus}}\\  
  \boldsymbol \beta  - \boldsymbol \alpha  + {{\boldsymbol{f}'}_B} & {\text{  if } \otimes \text{ is minus}} \\ 
\end{array}}
        \right.
\end{equation}
where $\boldsymbol \gamma  = \left( {{{{\alpha _1}} \mathord{\left/
	    			{\vphantom {{{\alpha _1}} {{\beta _1}}}} \right.
	    			\kern-\nulldelimiterspace} {{\beta _1}}},{{{\alpha _2}} \mathord{\left/
	    			{\vphantom {{{\alpha _2}} {{\beta _2}}}} \right.
	    			\kern-\nulldelimiterspace} {{\beta _2}}}, \cdots ,{{{\alpha _N}} \mathord{\left/
	    			{\vphantom {{{\alpha _N}} {{\beta _N}}}} \right.
	    			\kern-\nulldelimiterspace} {{\beta _N}}}} \right)$. 
From the lemmas, we know that, two vectors $\boldsymbol{f}'_A$ and ${\boldsymbol{f}'}_B$ are linearly related if they are in the relation of $\boldsymbol{f}'_A = k_1\cdot {\boldsymbol{f}'}_B +k_2$, where $k_1$ and $k_2$ are constant scalars, $k_1\neq0$. But, from the above discussion, the components of all the three vectors, $\boldsymbol \gamma$, $\boldsymbol \alpha  - \boldsymbol \beta$ and $\boldsymbol \beta  - \boldsymbol \alpha$, are not constant, due to the randomness of sample points $\left( {x_m^{\left( 1 \right)},x_m^{\left( 2 \right)},\; \cdots ,x_m^{\left( N \right)}} \right)$. This contradicts the supposition that the two vectors $\boldsymbol{f}'_A$ and ${\boldsymbol{f}'}_B$ are linearly related, and so Equation (\ref{assume}) cannot hold.
\end{proof}

\subsection{Notes on BiCT}
The proposed technique is called bi-correlation test (BiCT) since two complementary correlation tests are simultaneously carried out to determine whether a variable or a variable-combination is separable. 

The above process is illustrated with two sub-functions, and it could be extended to determine the separability of a function with more sub-functions. However, if the binary operators $\otimes_1, \otimes_2, \cdots, \otimes_m$ in equation (\ref{eq3}) are mutually different with mixed times and plus or minus, the extension might be a little difficult. This issue will be left for the future work. Hereafter, we assume that the binary operators in  equation (\ref{eq3}) are the same, i.e, $\otimes_i$= times or $\otimes_i$= plus or minus for all $i=1, 2, \cdots, m$, for simplicity. In this case, the extension process is very easy and omitted here.

To enhance the stability and efficiency of the algorithm, the distribution of sample points should be as uniform as possible. Therefore, controlled sampling methods such as Latin hypercube sampling \citep{Beachkofski} and orthogonal sampling \citep{Steinberg2006} are preferred for sample generation.

For the correlation test, any of correlation methods could be used. That is, Pearson's r method, Spearman's rank order correlation, and Kendall's $\tau$ correlation are all effective for BiCT.

Take the function $ f(\boldsymbol x)=0.8+0.6*(x_1^2+\cos(x_1))*\sin(x_2+x_3)*(x_2-x_3) , x \in [-3,3]^3$, as an example, the first sample set consists of 13 uniformly distributed points in $[-3,3]$, and the second sample set consists of 169 uniformly distributed points in $[-3,3]^2$. The correlation tests could be illustrated as in Fig. \ref{figure2}. As can be seen that the function-value vectors $\boldsymbol f^A$ and $\boldsymbol f^B$ are linearly related (Fig. \ref{fig2:b}), in which the variable $x_2$ and $x_3$ are fixed when considering the first variable $x_1$ (Fig. \ref{fig2:a}). Similarly, to find out the separability of variable combination ($x_2, x_3$), the first variable $x_1$ is fixed (Fig. \ref{fig2:c}). The corresponding function-value vectors $\boldsymbol f^C$ and $\boldsymbol f^D$ are linearly related (Fig. \ref{fig2:d})    

\begin{figure}
	\centering
	\subfigure[($x_2, x_3$) are fixed]{
		\includegraphics[width=0.47\linewidth] {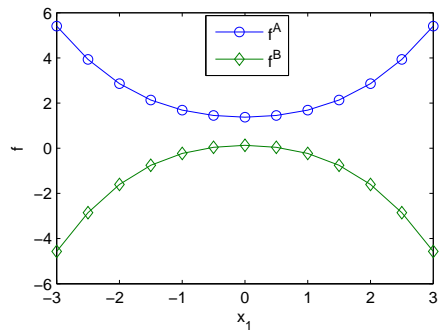}
		\label{fig2:a}
	}
	\subfigure[$\boldsymbol f^A\leftrightarrow \boldsymbol f^B$]{
		\includegraphics[width=0.47\linewidth] {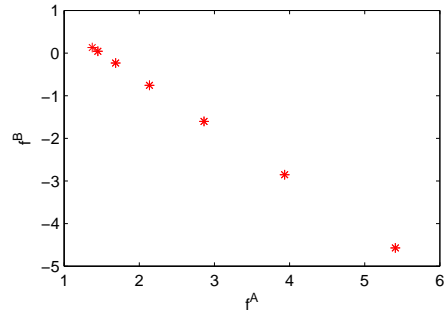}
		\label{fig2:b}
	}
		\subfigure[$x_1$ are fixed]{
			\includegraphics[width=0.47\linewidth] {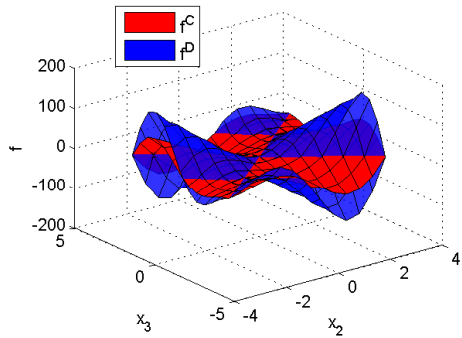}
			\label{fig2:c}
		}
		\subfigure[$\boldsymbol f^C\leftrightarrow \boldsymbol f^D$]{
			\includegraphics[width=0.47\linewidth] {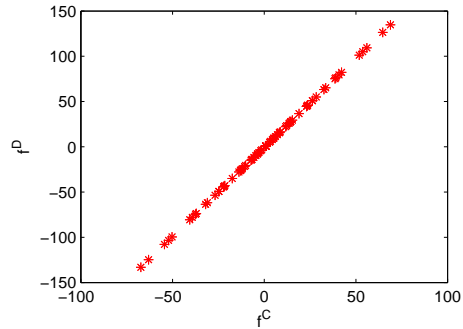}
			\label{fig2:d}
		}	
	\caption{Demo of separability detection process of BiCT}
	\label{figure2}
\end{figure}

The D\&C method with BiCT technique is described with functions of explicit expressions. While in practical applications, no explicit expression is available. In this case, some modifications need to adapt for D\&C method. In fact, for data-driven modeling problems, a surrogate model of black-box type could be established as the underlying target function \citep{Forrester2008} in advance. Then the rest steps are the same as above discussions.

\section{Numerical results}
\label{secNResults}
\subsection{Analysis on Computing time}

The computing time ($t$) of a genetic programming with the proposed divide and conquer (D\&C) method consists three parts: 
\begin{equation}
t=t_1+t_2+t_3
\label{eqTiming}
\end{equation}
where $t_1$ is for the separability detection, $t_2$ for sub-function determination, and $t_3$ for function recover. Note that both the separability detection and function recover processes are double-precision operations and thus cost much less time than the sub-function determination process. That is, $t \approx t_2$. 

It is obviously that the CPU time for determining all sub-functions ($t_2$) is much less than that of determining the target function directly ($t_d$). Next, a typical genetic programming, parse matrix evolution (PME), is taken as the optimization driver (other GP algorithms should also work) to show the performance of D\&C.

\begin{table}
\centering
\caption{A mapping table for parse-matrix evolution}%
\label{tabMapping}%
\begin{tabular}
[c]{c|ccccccccccc}\hline\hline
$a_{\cdot1}$ & -5 & -4 & -3 & -2 & -1 & 0 & 1 & 2 & 3 & 4 &5\\
T & $\sqrt{\cdot}$ & ln & cos & / & - & skip & + & * & sin & exp & $ (\cdot)^2$\\\hline
$a_{\cdot2},a_{\cdot3}$ & -5 & -4 & -3 & -2 & -1 & 0 & 1 & 2 & $\cdots$ & d & \\
expr & $\lambda_2$ & $\lambda_1$ & $f$ & $f_{2}$ & $f_{1}$ & 1.0 & $x_{1}$ & $x_{2}$ & $\cdots$
& $x_{d}$ & \\\hline
$a_{\cdot4}$ & -1 & 0 & 1 &  &  &  &  &  & \\
$f\rightarrow f_{k}$ & skip & $f_{1}$ & $f_{2}$ &  &  &  &  &  &\\
\hline\hline
\end{tabular}
\end{table}

Suppose that the dimension of the target function is $d$, the height of the parse-matrix $h$, and the mapping Table as in Table \ref{tabMapping}, then the parse-matrix entries $a_{\cdot 1} \in {-5,-4, ...,4,5}$, $a_{\cdot j}\in {-5, -4, \cdots,d}$, ($j=1,2$), and $a_{\cdot4}\in {-1, 0, 1}$. Thus the parse-matrix $(a_{ij})_{h \times 4}$ have as many as $(11*(6+d)*(6+d)*3)^h$ possible combinations. 
Thus the CPU time of determining the target function directly satisfies
\begin{equation}
t_d \sim (11*(6+d)*(6+d)*3)^h.
\label{eqtd}
\end{equation}
This means that the searching time of determining a target function will increase exponentially with model complexity. Using D\&C, only the sub-functions are needed to determinate, and each sub-function has less dimension $d$ and less complexity $h$. Therefore, it will cost much less CPU time. In fact, by D\&C powered GP, the CPU time will increase only linearly with dimensions, provided that the target function is completely separable. 

Take equation (\ref{equ2}) in Section \ref{secIntro} as an example. Without D\&C, to search directly, the control parameters of PME should be set as follows:  $d=5$, $h\geq 9$. From equation (\ref{eqtd}), the order of required CPU time $t_d = O(2.58\cdot 10^{32})$. 

Using D\&C powered PME, the required CPU time will be much less. In fact, after the separation detection, the function is divided into for sub-functions as follows.
\[{q_s} = 1.83 \times {10^{ - 4}}\boxed{{v^3}} \cdot \boxed{\sqrt \rho  } \cdot \boxed{1/\sqrt R } \cdot \boxed{(1 - {h_w}/{h_s})}\]
For the sub-function $v^3$, $1/\sqrt{R}$, and $1-h_w/h_s$, the control parameters of PME should be set as $d=1$, $h\geq 2$. For the sub-function $\sqrt{\rho}$, $d=1$, $h\geq 1$. As a result, $t_2 \approx 4* O(2.61\cdot 10^{6})=O(10^{7})$ by equation (\ref{eqtd}), which means the D\&C method could reduce the computational effort by orders of magnitude.

\subsection{Program timing}

Next, two illustrative examples are presented to show how much time the D\&C technique could save in practical applications.

Again, equation (\ref{equ1}) and equation (\ref{equ2}) are set as the target function, respectively. For equation (\ref{equ1}), the sample set consists 100 observations uniformly distributed in [1,10] degree and  [1000, 10000] (i.e., $\theta$=1:10; $Re_x$ = 1000:1000:10000). The angle $\theta$ is fixed to 5 degree while detecting the sub-function $f_1(Re_x)$, and the Renold number $Re_x$ is fixed to 5000 while detecting the sub-function $f_2(\theta)$. 

For equation (\ref{equ2}), the sample set consists 30000 observations uniformly distributed in a box in $R^5$ (i.e., $v=500:100:1000;
\rho=0.0001:0.0001:0.001;
R=0.01:0.01:0.1;
h_w=10000:10000:50000;
h_s=100000:100000:1000000$). 
The velocity of free-stream $v$, density of air $\rho$, radius of nose $R$, Wall enthalpy $h_w$, and total enthalpy $h_s$ are fixed to 800 $m/s^2$, 0.0005 kg/$m^3$, 0.05 m, 20000 J/kg, and 200000 J/kg, respectively, while detecting the sub-functions. 

In both tests, the program stops when the current model is believed good enough: $1-R^2<1.0\cdot 10^{-10}$, where $R^2=1-\frac{SSE}{SST}$ can be regarded as the fraction of the total sum of squares that has `explained by' the model. To suppress the affect of randomness of PME, 10 runs are carried out for each target function, and the averaged CPU time on a PC using a single CPU core (Intel(R) Core (TM) i7-4790 CPU @3.60GHz) is recorded to show its performance. The test results (see Table \ref{tab:eg1}, and Table \ref{tab:eg2}) show that D\&C technique can save CPU time remarkably. For equation (\ref{equ1}), PME needs about 12 minutes and 26 seconds to get an alternative of the target function without D\&C technique, while the D\&C powered PME needs only about 11.2 seconds, which is much faster than the original algorithm. Similar conclusion could also seen from the test results of equation (\ref{equ2}) (see Table \ref{tab:eg2}). Note that the total time of D\&C powered PME includes $t_1$, $t_2$ and $t_3$ (See Equ. \ref{eqTiming}), and $t_1+t_3\approx 0.2$ for Equ. (\ref{equ1}), 0.3 for Equ. (\ref{equ2}).

\begin{table}[htbp]
  \centering
  \caption{Performance of PME on detecting Equ. (\ref{equ1}) (with and without D\&C)} 
    \begin{tabular}{lll}
    \hline\hline
    Target Function & CPU time & Result Expression \\
    \hline
    \multicolumn{3}{l}{For D\&C powered PME} \\
    $f_1(Re_x)$ & 3s    & $St = 0.1978/\sqrt{Re_x}$ \\
    $f_2(\theta)$ & 8s$^\ast$    & $St = 0.03215*\theta - 0.01319*\theta^3$ \\
    Total time & 11.2s& \\
    \hline
     \multicolumn{3}{l}{PME without D\&C} \\
    $f(Re_x, \theta)$ & 12m26s$^\ast$   & $St = 2.274*\theta*\cos(0.9116*\theta)/\sqrt{Re_x} $\\
        Total time & 746s& \\
    \hline\hline
    \end{tabular}%
    \vspace{1ex}
   \\ $^\ast$ PME failed to get the exact result, but always result in an alternative function with fitting error of zero in double precision (i.e., $1-R^2=0.0$).
  \label{tab:eg1}%
\end{table}%

\begin{table}[htbp]
  \centering
  \caption{Performance of PME on detecting target Equ. (\ref{equ2})(with and without D\&C)}
    \begin{tabular}{lll}
    \hline\hline
    Target Function & CPU time & Result Expression \\
    \hline
    \multicolumn{3}{l}{For D\&C powered PME} \\
    $f_1(v)$ & 3s    & $q_s = 1.647\cdot 10^{-5}*v^3$ \\
    $f_2(\rho)$ & 2s    & $q_s =3.77\cdot 10^{5}*\sqrt\rho$ \\
    $f_3(R)$ & 4s    & $q_s = 6.49\cdot 10^{3}*\sqrt{0.08442/R} $\\
    $f_4(h_w, h_s)$ & 9s   & $q_s = 9370 -9370*h_w/h_s $\\
    Total time & 18.3s& \\
    \hline
    \multicolumn{3}{l}{PME without D\&C} \\
    $f(v, \rho, R, hw, hs)$ & 85m43s & $q_s = 0.000183*v^3*\sqrt{\rho/R} *(1-h_w/h_s)$ \\
     Total time & 5143s& \\
    \hline\hline
    \end{tabular}%
  \label{tab:eg2}%
\end{table}%

\section{Conclusion}
\label{SecConclusion}

The divide and conquer (D\&C) method for symbolic regression has been presented. The main idea is to make use of the separability feature of the underling target function to simplify the search process. In D\&C, the target function is divided into a number of sub-functions based on the information of separability detection, and the sub-functions are then determined by any of a genetic programming (GP) algorithms. 

The most important and fundamental step in D\&C is to identify the separability feather of the concerned system. To fulfill this task, a special algorithm, bi-correlation test (BiCT), is also provided for separability detection in this paper.  

The study shows that D\&C can accelerate the convergence speed of GP by orders of magnitude without losing the generality, provided that the target function has the feature of separability, which is usually the case in practical engineering applications.






\section*{References}
\bibliography{main}

\end{document}